\documentclass{article}

\usepackage[preprint]{neurips_2020}
\usepackage[utf8]{inputenc} 
\usepackage[T1]{fontenc}    
\usepackage{hyperref}       
\usepackage{url}            
\usepackage{booktabs}       
\usepackage{amsfonts}       
\usepackage{nicefrac}       
\usepackage{microtype}      
\usepackage{float}
\usepackage{times}
\usepackage{soul}
\usepackage{url}
\usepackage[utf8]{inputenc}
\usepackage{hyperref}
\usepackage[font =small]{caption}
\usepackage{graphicx}
\usepackage{subfigure}
\usepackage{mathpartir}
\usepackage{amsmath}
\usepackage{amsthm}
\usepackage{color}
\usepackage{booktabs}
\usepackage{algorithm}
\usepackage{algorithmicx}
\usepackage{algpseudocode}
\usepackage{dsfont}
\usepackage{amssymb}
\usepackage{wrapfig}
\usepackage{soul}
\setcitestyle{square,sort,comma,numbers}

\usepackage{color}

\newtheorem{example}{Example}
\newtheorem{theorem}{Theorem}
\newtheorem{definition}{Definition}

\newtheorem{lemma}{Lemma}

\title{Global Robustness Verification Networks}

 \author{%
   Weidi Sun\\
   School of Mathematical\\
   Peking University\\
   Haidian Qu 100871,Beijing, China \\
   \texttt{weidisun@pku.edu.cn} \\
   \And
   Coauthor \\
   Yuteng Lu\\
   School of Mathematical\\
   Peking University\\
   Haidian Qu 100871,Beijing, China \\
   \texttt{1701210111@pku.edu.cn} \\
   \And
   Coauthor \\
   Xiyue Zhang\\
   School of Mathematical\\
   Peking University\\
   Haidian Qu 100871,Beijing, China \\
   \texttt{zhangxiyue@pku.edu.cn} \\
   \And
   Coauthor \\
   Zhanxing Zhu\\
   School of Mathematical\\
   Peking University\\
   Haidian Qu 100871,Beijing, China \\
   \texttt{zhanxing.zhu@pku.edu.cn} \\
   \And
   Coauthor \\
   Meng Sun\\
   School of Mathematical\\
   Peking University\\
   Haidian Qu 100871,Beijing, China \\
   Center for Quantum Computing, Peng Cheng Laboratory\\
   Shenzhen Shi 518055, China \\
   \texttt{sunmeng@math.pku.edu.cn} \\
 }

\begin{document}

\maketitle
\begin{abstract}
 The wide deployment of deep neural networks, though achieving great success in many domains, has severe safety and reliability concerns. Existing adversarial attack generation and automatic verification techniques cannot formally verify whether a network is \emph{globally robust}, i.e., the absence or not of adversarial examples  in the input space.
 To address this problem, we develop a global robustness verification framework with three components: 1) a novel rule-based ``back-propagation'' finding which input region is responsible for the class assignment by logic reasoning; 2) a new network architecture \emph{Sliding Door Network (SDN)} enabling feasible rule-based ``back-propagation''; 3) a region-based global robustness verification (RGRV) approach. Moreover, we demonstrate the effectiveness of our approach on both synthetic and real datasets. 
\end{abstract}

\section{Introduction}\label{sec:intro}
Deep Neural Networks (DNNs) have been applied to a variety of domains and achieved great success. 
Reliance on DNNs' decisions makes their behavior reliability of high importance. 
Recent research has shown that the safety of DNNs is threatened by their susceptibility to human-imperceptible adversarial 
perturbations~\cite{propertie,perturbation,Biggio_weak}. 

To explore the adversarial robustness of neural networks, two aspects have been considered:  \emph{crafting adversarial examples} and \emph{automatic verification}.  
 Given an input sample, adversarial example generation techniques~\cite{FindAE2,Carlini017,PapernotMJFCS16,feature,SMT} 
fail to guarantee that no adversarial example exists around the given input, when they cannot generate adversarial examples for it.
The efforts in \emph{automatic verification}  mainly focus on the guarantee of 
local robustness~\cite{MirmanGV18,Huang,WuWRHK20,RuanHK18,GehrMDTCV18}, i.e., the robustness of an input's neighborhood.
These verification approaches can 
provide a rigorous local robustness proof if adversarial examples do not exist in a local region. 
However, the local robustness only takes a small part of the input space into account, and thus cannot guarantee reliability of the whole network for every possible input.  

Some attempts have been made towards the verification and evaluation
of \emph{global robustness}, i.e., finding out whether no adversarial example exists in the input space of a network  ~\cite{Hamming,SMT}.  
Though the SMT/SAT-based method in~\cite{SMT} takes global robustness into account, the definition for global robustness in~\cite{SMT} cannot be satisfied by the inputs near the classification boundary. In other words, no network can satisfy this definition.
The technique developed in~\cite{Hamming} evaluates the local robustness of each sample in a test dataset and 
treats the expected value of evaluation results as the indicator of ``global robustness''. The technique in~\cite{Hamming} can be considered as finding expected maximum safe radius over the test dataset.
Thus, the selection of the test dataset directly influences the estimation in~\cite{Hamming},
and the global robustness cannot be formally guaranteed in general. 
We can easily identify  two stumbling blocks on the path of global robustness verification: the \emph{complex activation patterns} and \emph{large input space}.
It is computationally prohibited to analyze all possible activation patterns or  traverse input space to guarantee global robustness.
Thus existing testing and verification techniques are infeasible to handle the global robustness verification for DNNs.


In this paper, 
we develop a feasible global verification framework with three components: 
1) a novel rule-based ``back-propagation'' which is used for
 mapping classification rules from output to input to find which input region is responsible for the corresponding class assignment\footnote{Note that this ``back-propagation'' is entirely different from typical use of back-propagation for evaluating the gradient with respect to the weight parameters in DNNs.}, 2) a new network design \emph{Sliding Door Network (SDN)} that enables feasible rule-based ``back-propagation'', 3) a region-based global robustness verification (RGRV) approach by finding ``adversarial regions''. 
Particularly, we address the ``two stumbling blocks'' by two means. 
Firstly, we design a new activation function \emph{Sliding Door Activation (SDA)}, with which the number of possible activation patterns 
is dramatically reduced to circumvent the complexity issue.  
Secondly, instead of treating a single input as the foundation ``atom'' of global robustness analysis, we cluster the input space into multiple classification regions to address the input space explosion challenge.  
To the best of our knowledge, this is the first work that can achieve global robustness formally with only slight drop of classification accuracy compared with classic DNNs.  
We evaluate the effectiveness of our framework on the MNIST~\cite{MNIST} dataset. We also design a synthetic case study to show the feasibility of our global verification method.

The rest of this paper is structured as follows. 
We introduce the rule-based ``back-propagation'' in Section~\ref{sec:rule}. 
The network design and the corresponding rule-based back-propagation method are described in Section~\ref{sec:SDN_SDNMAP}. 
Section~\ref{sec:globalrobustness} presents the RGRV approach. 
We evaluate the usefulness of SDN and effectiveness of RGRV in Section~\ref{sec:casestudy}. Section~\ref{sec:conclusion} summarizes our work.



\paragraph{Notations.}\label{sec:notations}
\begin{wrapfigure}[9]{r}{0.35\textwidth}
\vspace{-0.5cm}
\includegraphics[width=0.33\columnwidth]{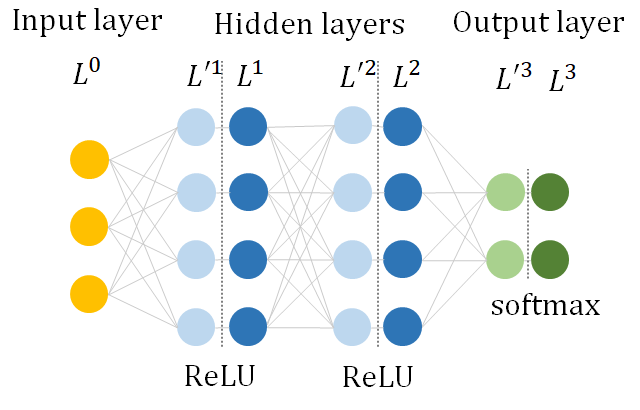}
\caption{A DNN with two hidden layers and one output layer. }
\label{Fig}
\end{wrapfigure}
For the convenience of presentation, each traditional layer \(Layer^h\) (\(0< h\leq H\)) is treated as two virtual layers:  pre-activation and activation layer, denoted by \({L'}^h\) and \({L}^h\), respectively.  
An example is shown in Figure~\ref{Fig}, where the activation layers are \(L^1,L^2,L^3\) and the pre-activation ones are  \({L'}^1,{L'}^2,{L'}^3\).
The \(i\)-th neurons in \(L^h\) and \({L'}^h\) are denoted as \(x^h_{i}\) and \({x'}^h_{i}\), respectively; the weights and the corresponding biases connecting \(L^{h-1}\) and \({L'}^h\) are represented as \(\omega^h_{ij}\) and \(b^h_j\); the activation function of \({L'}^h\) is \(f^h\), such as ReLU or softmax.

\section{Rule-based back-propagation for DNNs}\label{sec:rule}


Due to the colossal input space of DNNs, exhausting all possible inputs with traditional testing methods is infeasible. Thus we develop a family of classification rules to divide the input space into several regions. These regions can simplify the global robustness verification significantly. The classification rules in the input space could be achieved by the proposed rule-based back-propagation, as elaborated below. Before introducing the classification rules in detail, we first present a warm up example.

\begin{example}\label{example:warmup}
\emph{
Considering the one-layer network in Figure~\ref{fig_back:a}, a classification rule in the output space is (\(y_0 > y_1\)) which represents a blue region in output space shown as Figure~\ref{fig_back:b}.  The back-propagation we proposed aims for mapping classification rules to the input space. For example, the activation pattern ``all neurons are active'' means that (\({y}_0={y'}_0  \wedge {y}_1 = {y'}_1 \wedge {y'}_0 >0 \wedge {y'}_1 >0\)).
As (\({y}_0={y'}_0 \wedge {y}_1 = {y'}_1\wedge {y'}_0 = x_1 \wedge {y'}_1 = x_0\)),  (\(y_0 > y_1\)) is equivalent to (\(x_1 > x_0\)) and (\({y'}_0 >0 \wedge {y'}_1 >0\)) is equivalent to (\(x_1 > 0 \wedge x_0 >0\)). Thus the mapping result of output space (\(y_0 > y_1\)) to input space is (\(x_0 > 0 \wedge x_1 > 0 \wedge x_1 > x_0\)). If we change the activation pattern to ``\(y_0\) is active and \(y_1\) is inactive'', the equivalent condition of this activation pattern is 
(\({y}_0={y'}_0  \wedge {y}_1 = 0 \wedge {y'}_0 >0 \wedge {y'}_1 < 0\)), because  ReLU assigns 0 to \(y_1\). Thus (\({y'}_0 >0 \wedge {y'}_1 < 0\)) is equivalent to (\(x_1 > 0 \wedge x_0 < 0 \)); (\(y_0 > y_1\)) is equivalent to  (\(x_1 > 0\)); the mapping result is (\(x_0 < 0 \wedge x_1 > 0\)).  Obviously, the activation pattern determines the mapping result. The mapping result (\(x_0 > 0 \wedge x_1 > 0 \wedge x_1 > x_0\)) represents a blue region in input space shown in Figure~\ref{fig_back:c}. We name this blue region as classification region, indicating its responsibility to the class assignment.
}
\end{example}


\begin{wrapfigure}[17]{r}{0.45\textwidth}
\begin{minipage}[c]{0.5\linewidth}
\subfigure[one-layer network] { \label{fig_back:a}
\includegraphics[width=1\columnwidth]{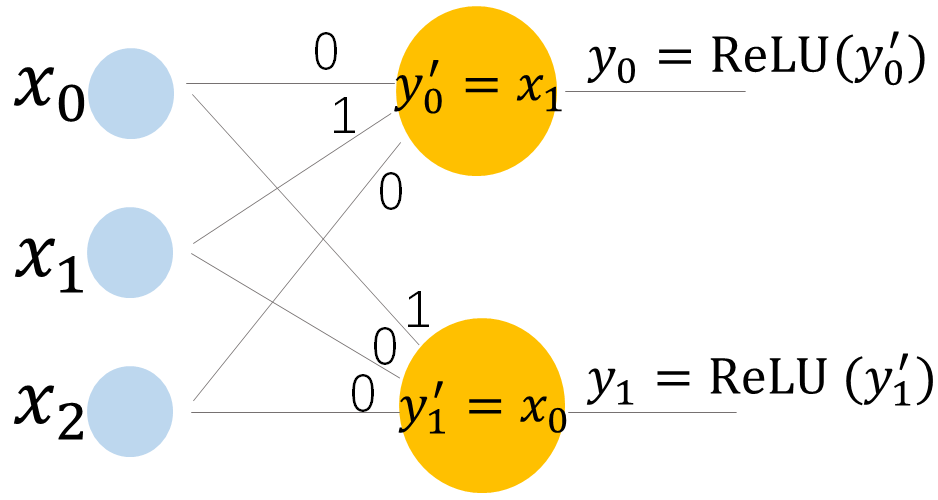}
}
\centering
\subfigure[region of (\(y_0 > y_1\))] { \label{fig_back:b}
\includegraphics[width=0.8\columnwidth]{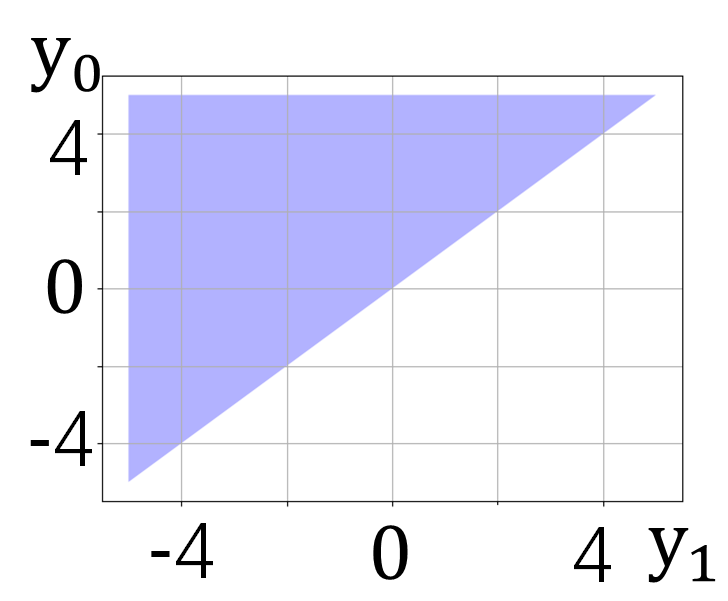}
}
\end{minipage}
\hspace{0in}
\begin{minipage}[c]{0.45\linewidth}
\subfigure[region of (\(x_0 > 0 \wedge x_1 > 0 \wedge x_1 > x_0\))] { \label{fig_back:c}
\includegraphics[width=1\columnwidth]{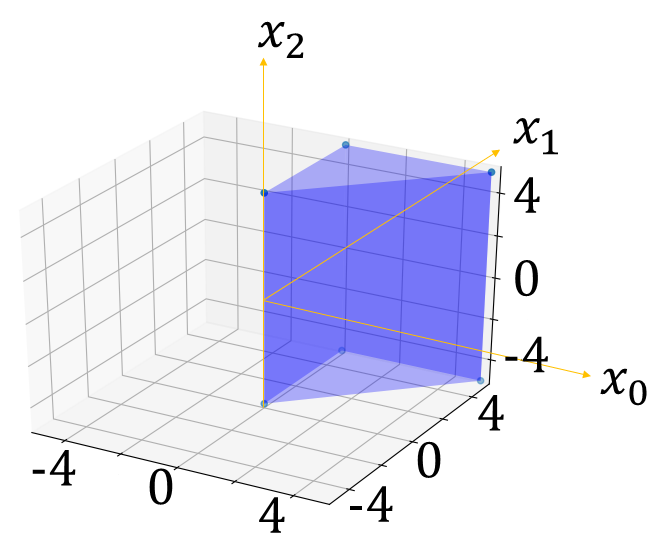}
}
\end{minipage}
\caption{The one-layer network and regions of classification rules.}
\label{fig_back}
\end{wrapfigure}
With the intuition from the warm up example, we now  elaborate the rule-based back-propagation layer by layer for deep neural networks.

There are many inequations recorded as \(\mathds{P}^h_{\gamma\eta}\) in \(Layer^h\). These inequations make up the disjunctive normal form \(\vee_\gamma\wedge_\eta\mathds{P}^h_{\gamma\eta}\) which describes how neural networks classify the inputs. For simplicity, if we select a \(\mathds{P}^h_{\gamma\eta}\) in \(Layer^h\) this \(\mathds{P}^h_{\gamma\eta}\) will be denoted as \(\mathds{P}^h = \mathop{\sum}\limits_i c_{i}x_i^{h}+b >0\).
We can easily provide the classification rules in output space. For example, ``the output belongs to class $k$'' is \(\mathop{\wedge}\limits_{j(j \neq k)} y_k>y_j\) where \(y_j\)s \((0\leq j<m)\) are the output values and \(m\) is the number of classes for output.
To obtain the classification rules in input space, we take a typical DNN\footnote{The activation functions in hidden layers and output layer of this DNN are ReLU and softmax, respectively.} as an example and propose a back-propagating function. 
During the back-propagation from \(Layer^h\) to \(Layer^{h-1}\), the \(\mathds{P}^h\)s should be substituted and the \(\wedge\)s and \(\vee\)s should be retained to \(Layer^{h-1}\), thus we apply our function to each \(\mathds{P}^h\) instead of the conjunctive normal form. The recursive call of this function can back-propagate the classification rules layer by layer to the input space of the network. 
Since the output layer and the hidden layers should be treated in different ways, we divide the function into two parts: the output and hidden layer part. 

The comparison rules like \(y_k > y_j\) can be directly mapped to the corresponding pre-activation layer based on the order-preserving activation function, e.g., softmax, of the output layer. We replace every variable \({y}_j\) in inequations with the corresponding polynomial \(\sum_i \omega_{ij}^H x_i^{H-1} + b_j^H\)~to obtain the classification rule in \(Layer^{H-1}\). Thus the output layer part is the function
\begin{equation}\label{eq:MAP-OUT2}
\textstyle  \emph{MAP-OUT}({y}_k>{y}_j) = \sum_i\omega_{ik}^H x_i^{H-1} + b_k^H> \sum_i \omega_{ij}^H x_i^{H-1}+b_j^H
\end{equation}
and the mapping result of \(\mathop{\wedge}\limits_{j(j \neq k)} y_k>y_j\) is \(\mathop{\wedge}\limits_{j(j \neq k)} \emph{MAP-OUT}({y}_k>{y}_j)\), i.e., the classification rules in \(Layer^{H-1}\).

The hidden layer part is a function \emph{MAP-HIDDEN}. Since each linear inequation in \(L^h\) can be simplified into the form \(\mathop{\sum}\limits_i c_{i}x_i^{h}+b >0\), we select an inequation \(\mathds{P}^h = \mathop{\sum}\limits_i c_{i}x_i^{h}+b >0\) as the input of \emph{MAP-HIDDEN} to show how \emph{MAP-HIDDEN} works. The hidden layers cannot be processed in the same way as output layer because of the activation patterns which determine the mapping result.
We denote the set of active neurons' indexes by \(\Theta\) and use \(\Theta\) to represent the activation pattern. For simplicity, we record the mapping result of \(\mathds{P}^h\) under \(\Theta^h\) as \emph{MAP-FIX\((\Theta^h,\mathds{P}^h)\)} where \(\Theta^h\) is the activation pattern of \(Layer^{h}\). \emph{MAP-FIX\((\Theta^h,\mathds{P}^h)\)} is the conjunction of some classification rules in \(Layer^{h-1}\).
The function \emph{MAP-HIDDEN} is shown as follows where \(\Delta^h\) denotes all activation patterns of \(Layer^{h}\):
\begin{equation}\label{eq:MAP-HIDDEN}
\textstyle  \emph{MAP-HIDDEN}(\Delta^h,\mathds{P}^h)= \mathop{\vee}\limits_{ \Theta^h\in \Delta^h}\textstyle\emph{MAP-FIX\((\Theta^h,\mathds{P}^h)\)}
\end{equation}
The mapping result of \emph{MAP-HIDDEN} is the disjunction of all the classification rules in \(Layer^{h-1}\).
As each neuron has two activation states, there are \(2^{m^h}\) activation patterns in \(\Delta^h\) where \(m^h\) is the number of neurons in \(Layer^{h}\). The time cost of whole  back-propagation is \(O(\Pi_{h} 2^{m^h})= O(2^{\sum_h m^h})\). Such immense time cost makes the above mapping NP-hard and infeasible.

\section{Sliding Door Network for Feasible Back-propagation}\label{sec:SDN_SDNMAP}

To handle the complexity issue, we present a novel network design, SDN, and the corresponding rule-based back-propagation method \(M_{SDN}\). SDN reduces the size of \(\Delta\) by grouping the neurons in each layer to overcome the infeasibility problem in back-propagating classification rules for DNNs.


\subsection{Sliding Door Network}\label{sub_sec:SDN}
\begin{wrapfigure}[10]{r}{0.6\textwidth}
\includegraphics[width=0.6\columnwidth]{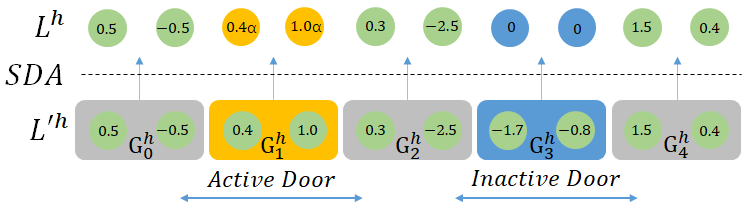} 
 \caption{Sliding Door Activation. \(G_1^h\) and \(G_3^h\) are active door and inactive door respectively and other activation results are the copy of other groups.}
\label{fig:SDA}
\end{wrapfigure}
Compared with typical DNNs, SDN has two different components: a novel  activation function SDA and the loss function design for supporting SDA.

\textbf{Sliding Door Activation.}  SDA takes a  pre-activation layer into account and divides neurons into several groups evenly. For example, the layer \({{L'}^h}\) in Figure~\ref{fig:SDA} with 10 neurons is divided into 5 groups which are represented as \(G^h_j = \{ {x'}_i^h| 2j \leq i < 2j + 2 \}\ (0\leq j < 5)\). These groups are classified by SDA into three categories: \emph{active group} with all positive neurons (e.g., \(G_1^h\) and \(G_4^h\) in Figure~\ref{fig:SDA}), \emph{inactive group} in which all neurons are negative (e.g., \(G_3^h\) in Figure~\ref{fig:SDA}), and \emph{trivial groups} with mixing of both positive and negative neurons (i.e., \(G_0^h\) and \(G_2^h\) in Figure~\ref{fig:SDA}).  

In order to reduce the complexity, we select the first active (inactive) group as \emph{active} (\emph{inactive}) \emph{door} for each pre-activation layer. For example in Figure~\ref{fig:SDA}, \(G_1^h\) and \(G_3^h\) are active door and inactive door respectively. Based on the assigned doors, we define SDA as: 
\begin{align}
x^h_i = SDA(x'^h_i)
=&\begin{cases}
0 &   \text{if }  x'^h_i \text{ belongs to inactive door};\\
\alpha x'^h_i &  \text{if }  x'^h_i \text{ belongs to active door}; \\
 x'^h_i &  \text{otherwise}.
\end{cases}
\end{align}

\begin{wrapfigure}[14]{r}{0.3\textwidth}
\includegraphics[width=0.3\columnwidth]{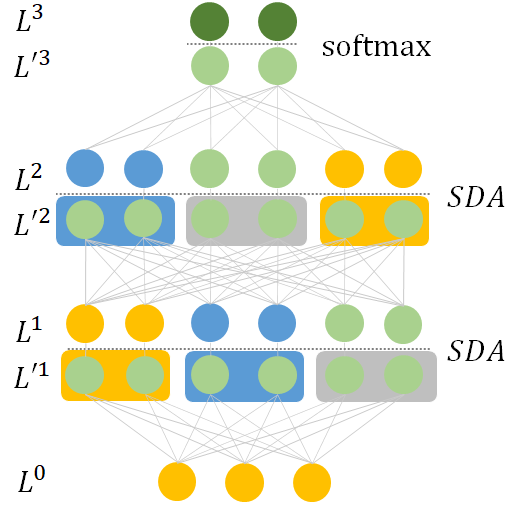} 
\caption{Architecture of SDN}
\label{fig:SDN}
\end{wrapfigure}
To increase the network expressiveness, SDA strengthens the active door by \(\alpha \) and  assigns \(0\) to inactive door's neurons. Other groups are sent to the corresponding activation layer directly. During the training, for each pre-activation layer, the position of the two doors might change instantly up to the states of the groups, behaving like a sliding door, thus the name of 
our activation function.  Figure~\ref{fig:SDN} shows the entire network architecture, replacing the ReLU in classic DNNs with the proposed SDA for each layer.

\textbf{Loss function design.} If a pre-activation layer cannot provide \emph{active} or inactive door, the expressiveness of SDN will be weakened. To avoid this issue, we design regularization term to penalize the absence of either of the two doors. If the active (inactive) door does not appear in \({L'}^h\), we will find the group \(G_{\alpha}^h\) (\(G_{\beta}^h\)) in \({L'}^h\) with most active (inactive) neurons, and adjust the weights to make the negative (positive) neurons in \(G_{\alpha}^h\) (\(G_{\beta}^h\)) tend to be positive (negative) so as to create active (inactive) groups.
Thus, besides the typical data fitting loss, we add a regularization term to encourage the emergence of such groups, defined as: 
\begin{align}
L(W,b) = \sum_{i=1}^n  (y_i-\hat{y}_i)^2  + \lambda \sum_h (\sum_{x_i'^h \in G_{\alpha}^h, x_i'^h < 0} (-x_i'^h) + \sum_{x_i'^h \in G_{\beta}^h, x_i'^h >0 } x_i'^h)
\end{align}
where $\{ W, b\}$ denotes all the weights and biases to be trained, and  \(\lambda\) is the user-given penalty parameter.

\subsection{Rule-based Back-propagation for SDN}\label{sub_sec:sdamap}
As \emph{MAP-OUT} can be reused for the back-propagation of SDN's output layer, we focus on back-propagation between hidden layers in this section. The constructing process of \emph{MAP-HIDDEN} for SDN is as follows.

We denote the set of neurons in the active door of layer \({L'}^h\) as \(\Theta^h_A\), the set of neurons in the inactive door as \(\Theta^h_I\), and other neurons are in \(\Theta^h_T\).
Considering the condition that a rule in \(L^h\) is \(\mathds{P}^h = \mathop{\sum}\limits_i c_ix_i^{h}+b >0\),  and the activation pattern \(\Theta^h\) is fixed where \(\Theta^h = \{\Theta^h_A, \Theta^h_I, \Theta^h_T\}\),  we record the mapping result of \(\mathds{P}^h\) under \(\Theta^h\)  as MAP-FIX\((\Theta^h,\mathds{P}^h )\) with three components:

1) \(SDNInherit\). 
\begin{equation}
    SDNInherit=\textstyle{\mathop{\sum}\limits_{i\in \Theta^h_A}\!\!\!\! \alpha c_i((\mathop{\sum}\limits_t\omega_{ti}^{h} x_t^{h-1})+ b_i^{h})+\!\!\!\!\mathop{\sum}\limits_{i\in \Theta^h_T}\!\!\!\!c_i((\mathop{\sum}\limits_t \omega_{ti}^{h} x_t^{h-1}) + b_i^{h})+b>0},
\end{equation}
where we replace the neurons in \(\mathds{P}^h\) belonging to \(\Theta^h_A\) with the corresponding polynomial, i.e., \( \sum_t \omega_{ti}^h x_t^{h-1} +b_i^h\) multiplied by \(\alpha\) due to SDA activation, the neurons in \(\Theta^h_T\) with the corresponding polynomial meanwhile remove the  neurons in \(\Theta^h_I\) to obtain \(SDNInherit\).

2) \(SDNActiveCon\). 
\begin{equation}
    SDNActiveCon = \textstyle{\mathop{\bigwedge}\limits_{i \in \Theta^h_A}((\mathop{\sum}\limits_t \omega_{ti}^{h} x_t^{h-1}) + b_i^{h-1} > 0)}
\end{equation}
It describes the rules that ``all the corresponding pre-activation neurons of \(\Theta^h_A\) are greater than 0'' and we replace these pre-activation neurons with corresponding polynomial.
 
3) \(SDNInactiveCon\). 
\begin{equation}
    SDNInactiveCon = \textstyle{\mathop{\bigwedge}\limits_{i \in \Theta^h_I}((\mathop{\sum}\limits_t -\omega_{ti}^{h} x_t^{h-1}) - b_i^{h} > 0)}
\end{equation}
It describes the rules that ``all the corresponding pre-activation neurons in \(\Theta^h_I\) are less than 0'' and we replace these pre-activation neurons with corresponding polynomial.
The function \emph{MAP-FIX} is shown as follows:
\begin{equation}\label{eq:SDN-FIX}
 \textstyle\emph{MAP-FIX\((\Theta^h,\mathds{P}^h )\)} = {\emph{SDNInherit}\ \wedge\ \emph{SDNActiveCon}\ \wedge\  \emph{SDNInactiveCon}}
\end{equation}
Taking all the activation patterns into account,  we can obtain the function \emph{MAP-HIDDEN}. 
\begin{equation}
\textstyle  \emph{MAP-HIDDEN}(\Delta^h,\mathds{P}^h)= \mathop{\vee}\limits_{ \Theta^h\in \Delta^h}\textstyle\emph{MAP-FIX\((\Theta^h,\mathds{P}^h)\)}
\end{equation}
The combination of \emph{MAP-OUT} and \emph{MAP-HIDDEN} is the complete back-propagating  function \(M_{SDN}\):
\begin{align}
M_{SDN}(\Delta^{h},\mathds{P}^h)
=&\begin{cases}
\emph{MAP-OUT}(\mathds{P}^h) &  h=H\\
 \emph{MAP-HIDDEN}(\Delta^{h}, \mathds{P}^h) &  h<H
\end{cases}
\end{align}
Thus the mapping result of all the rules \(\vee_\gamma\wedge_\eta\mathds{P}^h_{\gamma\eta}\) in \(Layer^h\) is  \(\vee_\gamma\wedge_\eta M_{SDN}(\Delta^{h},\mathds{P}^h_{\gamma\eta})\) which is the  collection of  rules for \(Layer_{h-1}\).  
Each \(|\Delta^h|\)  equals to \({m_h}(m_h -1)+ 2m_{h} + 1\) where \(m_{h}\) is the number of groups in \({L'}^h\). The SDN has \(O(\Pi_{i} {m_i}^2)\) activation patterns which is greatly less than the number of DNN's activation patterns \(O(2^{\sum_h m^h})\) and its rule-based back-propagation becomes more feasible. 

\(M_{SDN}\) maps the \emph{explicit rules}, but ignores the \emph{implicit rules}. \emph{Implicit rules} are the constraints from \emph{trivial groups} guaranteeing that compared with other active (inactive) groups in \(Layer^h\) the active (inactive) door \(G_i^h\) has the minimal \(i\). For example in Figure~\ref{fig:SDA}, the \emph{implicit rules} are:
\begin{align*}
\exists_{{x'}_i^h\in G^h_0} {x'}_i^h < 0 \wedge
\exists_{{x'}_i^h\in G^h_0} {x'}_i^h > 0 \wedge
\exists_{{x'}_i^h\in G^h_2} {x'}_i^h > 0
\end{align*}
The combination of \emph{explicit rules} and \emph{implicit rules} can be organized into the form \(\vee_\gamma\wedge_\eta\mathds{P}^h_{\gamma\eta}\). Each \(\wedge_\eta\mathds{P}^h_{\gamma\eta}\) represents a classification region like the blue region in Figure~\ref{fig_back:c}. Selecting a \(\mathds{P}^h = \mathop{\sum}\limits_i c_ix_i^{h}+b >0\) in  \(\wedge_\eta\mathds{P}^h_{\gamma\eta}\), a boundary of \(\wedge_\eta\mathds{P}^h_{\gamma\eta}\)'s classification region is  \( \mathop{\sum}\limits_i c_ix_i^{h}+b = 0\).
We will show that \emph{explicit rules} is sufficient for global verification in Theorem~\ref{theorem:common}.
\section{Region-based Global Robustness Analysis  }\label{sec:globalrobustness}

In this section, we define and address the global robustness verification problem by a region-based global robustness analysis (RGRA) approach.
Firstly, we provide the definition for global robustness:
\begin{definition}[Global robustness]\label{def:globalrobustness}
Given a network N, if there is no adversarial example in its input space, N is a globally robust network.
\end{definition}

An adversarial example exists in two types of regions: 1) the region which is isolated, 2) the region which is connected to the correctly classified region. Taking Figure~\ref{fig_ill} as an example, it shows a binary classification task where the black dashed line is orale decision boundary of class \(C_1\) and \(C_2\) and the orange solid line is the decision boundary determined by network. 
\begin{wrapfigure}[10]{r}{0.35\textwidth}
\includegraphics[width=0.35\columnwidth]{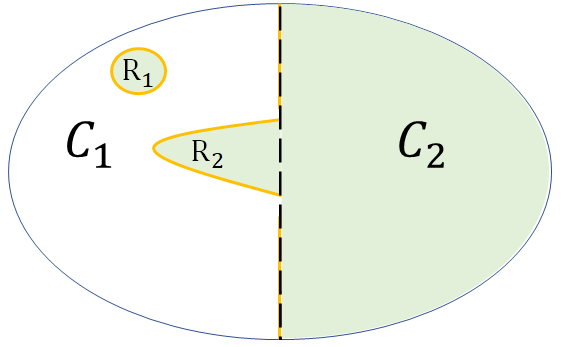}
\caption{Small-size isolated connected component and protruding region. \(R_1\) and \(R_2\) are the adversarial regions.}
\label{fig_ill}
\end{wrapfigure}
\(R_1\) and \(R_2\) are adversarial regions. \(R_1\) is the \emph{small-size isolated connected component} and \(R_2\) is the \emph{protruding region} which is connected to the correctly classified region. The adversarial examples belonging to \(C_1\) but classified as \(C_2\) tend to exist in \(R_1\) and \(R_2\).
Before presenting the definitions of \emph{protruding region} and \emph{small-size isolated connected components},
we first present a formal definition of the input space classification graph.

\begin{definition}[Classification graph]\label{def:graph}
Given a result of backward-mapping in the form of \(\vee_\gamma\wedge_\eta\mathds{P}^h_{\gamma\eta}\), we can build a \emph{classification graph} as a tuple \((V,E)\) where
\begin{itemize}
    \item \(V = \{v_i\ |\ v_i\ is\ the\ \wedge_\eta\mathds{P}^h_{\gamma\eta}\ in\ \vee_\gamma\wedge_\eta\mathds{P}^h_{\gamma\eta}\}\). In other words, each classification region can be defined as a vertex \(v_i\).
    \item \(E = \{(v_i,v_j)\ |\ v_i\ and\ v_j\ are\ adjacent\ vertexes\}\). Given two \emph{vertexes} \(v_i\) and \(v_j\), the corresponding classification regions in the input space are \(S_i\) and \(S_j\). Two \emph{vertexes} are adjacent iff their classification regions are adjacent in high-dimensional space, and formally, 
    \begin{align}
    \exists_x\exists_{\varepsilon>0}\{[\forall_{x'\in\ I}( x'\in S_i \vee x'\in S_j)]
    \wedge  (I \cap S_i \neq \emptyset) \wedge  (I \cap S_j \neq \emptyset)\}
    \end{align}
    where \(I =\{x'|\ ||x'-x||_\infty<\varepsilon\}\). 
\end{itemize}
\end{definition}

\begin{definition}[Limiting ball]\label{def:limit_ball}
Given a set of vertexes \(V'\), the stitching of their regions is \(S\) whose center of gravity is \(o\). The limiting ball \(B\) of \(V'\) can be defined as a ball whose center is \(o\) and radius \(d = max(\{dis|dis= ||x-o||_\infty,  x\in S\})\).
\end{definition}
With the description of classification graph and  limiting ball, we now provide the formal definition of two types of regions in which adversarial example exist. 
\begin{definition}[Small-size isolated connected component]\label{def:isolate}
Given a connected component \(S\) in classification graph and a \(\mathcal{R}\) (\(\mathcal{R} > 0\)), let \(B\) be \(S\)'s vertexes' limiting ball. If \(B\)'s radius is smaller than \(\mathcal{R}\), \(S\) is a small-size isolated connected component. 
\end{definition}
\begin{definition}[Protruding regions]\label{def:protruding}
Given a vertex \(v\) and a \(r\) (\(r > 0\)), let \(B\) be \(\{v\}\)'s limiting ball. The classification region of \(v\) is \(S\).  All points with same class as \(S\) in \(B\) constitute a set \(Cl(S)\). The volume of \(Cl(S)\) and \(B\) are recorded as \(vol(Cl(S))\) and \(vol(B)\) respectively.  If \(S\) is not a small-size isolated connected component and \(vol(Cl(S)) < r\times vol(B)\), \(S\) is a protruding region. 
\end{definition}
The first step of finding the ``adversarial regions'' is to construct the adjacency relationship between vertexes, i.e., building \(E\) in classification graph. The construction can be split into two phases: 1) traverse the vertexes \(v_i\) in \emph{classification graph}, and treat \(v_j\) (\(j\in \{t | t > i\}\)) as the potential adjacent vertexes; 2) find the common boundary shared by \(v_i\) and \(v_j\) if they are adjacent, and provide the proof that they are not adjacent otherwise. Intuitively, a common boundary between \(v_i\) and \(v_j\) means that they stick together on the boundary, formally defined as follows.
\begin{definition}[Common boundary]\label{def:common}
 Given two \emph{vertexes} \(v_i\) and \(v_j\) in the classification graph, the corresponding classification regions are \(S_i\) and \(S_j\). A boundary \(\mathbb{F}\) (\(\sum c_ix^0_i+b = 0\)) which belongs to \(v_i\) and \(v_j\) is defined as a common boundary iff
\begin{align}
    \exists_{x_0}\exists_{\varepsilon_0>0} [ \forall_{\mathbb{H'}\neq \mathbb{H}}\ I \cap \mathbb{H'} = \emptyset\ \wedge\ \forall_{x\in I}\ ( x\in S_i \vee x\in S_j)
    \wedge  I \cap S_i \neq \emptyset \wedge I \cap S_j \neq \emptyset ]
\end{align}
 where \(\mathbb{H} = \{x|\sum c_ix^0_i+b = 0\}\) is the set of points on \(\mathbb{F}\), \(\mathbb{H'}\) is the set of points on other boundaries \(\mathbb{F'}\) of \(v_i\) and \(v_j\), and \(I = \{x|\ ||x-x_0||_\infty<\varepsilon_0\}\).
\end{definition}

As the \emph{classification graph} is an undirected graph, the first phase only explores the edge \((v_i,v_j)\)  \( (j>i) \) to find the potential adjacent vertex \(v_j\). And the following 
Theorem~\ref{theorem:common} implies that by traversing the \(explicit\ rules\) of \(v_i\), we can find out all the \(v_j\) (\(j\in \{t | t > i\}\)) which share a common boundary with \(v_i\), i.e., all the adjacent vertexes.
\begin{theorem}\label{theorem:common}
Given \(v_i\) and \(v_j(j>i)\), there is a common boundary belonging to the  \(explicit\ rules\)  of \(v_i\) shared by them iff \(v_i\) and \(v_j\) are adjacent.
\end{theorem}
Proof: see appendix.
Finally, the global robust verification can be achieved by the following steps:
\begin{enumerate}
  \item Build classification graph. We can obtain the vertexes from the back-propagation result. Each vertex has some boundaries. For each boundary \(\mathbb{F}\) of \(v_i\), we select some points on it randomly and sample in the tiny neighborhood of these points. By feeding the samples into SDN, we can find out their activation patterns. For example, a sample's activation pattern represents \(v_j\), as this sample is in the tiny neighborhood of points on \(\mathbb{F}\), \(\mathbb{F}\) belongs to 
  \(v_j\). Thus \(v_i\) and \(v_j\) shares the boundary \(\mathbb{F}\) and we add edge \((v_i, v_j)\) to the edge set \(E\). 
 \item Find the limiting ball of classification regions and  connected components. For a  classification region \(S\), we can obtain the rough upper and lower bound of each dimension. Taking the inequation  \(\mathop{\sum}\limits_i c_ix_i^{0}+b >0\) as an example, if \(c_1>0\) and \(-2\leq x_i \leq 2\), the lower bound of \(x^0_1\) is \(x^0_1>-((\mathop{\sum}\limits_{i(i\neq 1\wedge c_i>0)}\!\!\!\! 2c_i + \!\!\!\!\!\mathop{\sum}\limits_{i(i\neq 1\wedge c_i<0)}\!\!\!\!\!-2c_i)+b)/c_1\). All these bounds form a ``box'' including \(S\). We take samples \(u_t\) (\(0\leq t<n\)) in this ``box" and select the samples in \(S\). \(o = \sum_t u_t/n\) is the estimated center. Denote \(d(u_t,o)\) as the distance between \(u_t\) and \(o\), the estimated radius is \(\max_t(d(u_t,o))\). Moreover, we can calculate the volume of the ``box'' \(vol(box)\), if there are \(m\) samples in \(S\),  \(vol(s) =\frac{m}{n}vol(box)\).  If a connected component consists classification regions \(S_k\)s, the estimated center is \(o = \frac{\sum_k o_k vol(S_k)}{\sum_k vol(S_k)}\) and \(\max_{t,k}(d(u_{kt},o))\) is the estimating radius where \(u_{kt}\)s are the samples in \(S_k\).
 
  \item Find small-size isolated connected components and protruding regions. Given a connected component, by comparing the radius of limiting ball with the \(\mathcal{R}\) given by user, we can find out whether a connected component is a small-size isolated connected component. For a classification region belonging to class \(C\), we take \(n\) samples in the limiting ball. If \(m\) of them belong to \(C\), calculate \(n/m\) and compare it with the user given \(r\). Then we can find out whether this region is a protruding region.
\end{enumerate}

\section{Experiments}\label{sec:casestudy}
The evaluation of our work concentrates on three aspects: 1) the utility of SDN on classification tasks, 2) generating adversarial examples, 3) the feasibility of global verification. In the first part, we evaluate our method on \emph{MNIST} database with typical DNNs as baseline. In the second part, we show the adversarial examples generated from adversarial regions. In the third part, we design a synthetic case study to show the effectiveness of RGRV.

\subsection{Utility on classification task}\label{subsec:utility}
Due to the reduction of activation patterns, the expressive ability of SDN is slightly inferior to typical DNNs with the same architecture. We empirically show the drop of classification accuracy is acceptable based on two groups of case studies. The first group is the comparison of typical DNNs and SDNs on \emph{MNIST} dataset.
The details of SDNs are shown as follows. 

\begin{table}[h]
\vspace{-0.3cm}
\caption{Classification results on MNIST where sat-rate is the frequency of layers which can provide both active door and inactive door in evaluation}
\vspace{-0.2cm}
  \label{table1}
  \centering
  \vspace{11pt}
    \begin{tabular}{l|l|l|l|l|l}
       MNIST&&(16,12)&(24,18)\ &(32,24)\ &(40,30)\\\midrule\midrule
     SDN& accuracy(\%)& 95.08& 95.24& 95.32& 95.45\\
              & sat-rate(\%)& 81.64& 85.25& 93.49& 96.37\\
     DNN      &accuracy(\%) &97.80 &  98.12& 98.21&98.28\\ \midrule\midrule
    \end{tabular}
    \vspace{-0.2cm}
\end{table}

   1) Each SDN has two hidden layers \(Layer^1\) and \(Layer^2\);
   2) These \emph{SDNs} have 16, 24, 32, 40 groups in \(Layer^1\), respectively and 12, 18, 24, 30 groups in \(Layer^2\),  respectively. We name these SDNs as (16,12), (24,18), (32,24), and (40,30) based on their architecture features;
   3) Each group in \(Layer^1\) has four neurons, and each group in \(Layer^2\) has two neurons.
   4) The \(\alpha\) in these SDNs are set as 2.

Each baseline DNN has two hidden layers. The number neurons in each layer is the same as corresponding \emph{SDNs}. Cross entropy loss and Adam~\cite{Adam} are used to train all the networks 1500 epochs with batch size 256. The evaluation result is shown in Table~\ref{table1}.
Compared with typical DNNs, the accuracy of SDN only drops 2.72, 2.88, 2.89, 2.83 percent respectively. Besides, the accuracy of SDN and the \emph{sat-rate} increase as the numbers of groups in each layer increase.

\subsection{Generating adversarial examples}
Since our verification method is aimed at global verification, i.e., finding all the adversarial regions in the input space, the generation of attacks is only a by-product. As this generation is not  based on local or testing way, it is meaningless to compare its efficiency with state-of-the-art attack generation approaches like~\cite{FindAE2,Carlini017,PapernotMJFCS16,feature,SMT}.  Given parameters\((\mathcal{R},r) = (0.04, 0.2)\), we show the adversarial examples in SDN (20,20)\footnote{(20,20) has 20 groups in \(Layer^1\) and \(Layer^2\) and each group has three neurons. (20,20) is trained on MNIST images which are resized as \(4\ \times\ 4\). All the other settings are same as the \emph{SDNs} in subsection~\ref{subsec:utility}.} and point out the corresponding adversarial regions.

\begin{wrapfigure}[11]{r}{0.3\textwidth}
\centering
\begin{minipage}[c]{0.18\linewidth}
\subfigure {
\includegraphics[width=1\columnwidth]{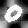}
}\vspace{-0.11in}
\subfigure {
\includegraphics[width=1\columnwidth]{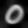}
}
\end{minipage}
\begin{minipage}[c]{0.18\linewidth}
\subfigure {
\includegraphics[width=1\columnwidth]{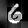}
}\vspace{-0.11in}
\subfigure {
\includegraphics[width=1\columnwidth]{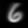}
}
\end{minipage}
\begin{minipage}[c]{0.18\linewidth}
\subfigure {
\includegraphics[width=1\columnwidth]{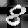}
}\vspace{-0.11in}
\subfigure {
\includegraphics[width=1\columnwidth]{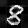}
}
\end{minipage}
\begin{minipage}[c]{0.18\linewidth}
\subfigure {
\includegraphics[width=1\columnwidth]{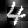}
}\vspace{-0.11in}
\subfigure {
\includegraphics[width=1\columnwidth]{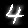}
}
\end{minipage}
\begin{minipage}[c]{0.18\linewidth}
\subfigure {
\includegraphics[width=1\columnwidth]{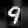}
}\vspace{-0.11in}
\subfigure {
\includegraphics[width=1\columnwidth]{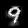}
}
\end{minipage}
\caption{Adversarial examples are in the first line and the samples in second line are classified correctly}
\label{fig:attack}
\end{wrapfigure}
The digits in the first line of Figure~\ref{fig:attack} classified as ``2'' are the adversarial examples of the digits below them classified as ``0, 6, 8, 4, 9'' respectively. 
Take \(4\) as an example, the upper \(4\) is in ``[[18,1],[1,15]]'', denoting the activation pattern  where \(G_{18}^1\) and \(G_1^2\) are active doors, \(G_1^1\) and \(G_{15}^2\) are inactive doors. By inputting ``[[18,1],[1,15]]'', we can find a group of inequations returned by our algorithm, and the conjunction of these inequations represents the corresponding classification region of ``[[18,1],[1,15]]''.
The lower \(4\)  classified as \(4\) is in the activation pattern ``[[4,10],[1,3]]''. The classification region of ``[[18,1],[1,15]]'' is the \emph{protruding region} found by our method which is close to ``[[4,10],[1,3]]'s'' classification region. Moreover, we have found out that there is only one connected component of class 2 in the input space. Obviously, it is easy to select a large number of adversarial examples in the adversarial regions.

\subsection{Feasibility of precise global verification}
As our work is the first global robustness verification work, there is no baseline method existing for this case study. To draw an exact conclusion whether the results of the proposed method are correct, we train a SDN on a two-dimensional synthetic dataset shown in Figure~\ref{sub_fig:dataset}. In Figure~\ref{sub_fig:dataset}, the big blue region at the top-right belongs to the first class and the small blue region at the lower-left which is the ``noise'' in this dataset belongs to the first class as well. The points in the white region belong to the second class. The classification results of the trained SDN are visually shown in
Figure~\ref{sub_fig:classifyrule} where the points in the blue regions belong to the first class and the points in the white region belongs to the second class.
Our method has found the adversarial regions in the orange circles. Obviously, these regions containing adversarial examples are what we do not want. The verification result shows that this SDN is not globally robust.

\section{Conclusion}\label{sec:conclusion}
\begin{wrapfigure}[11]{r}{0.44\textwidth}
\subfigure[synthetic dataset] { \label{sub_fig:dataset}
\includegraphics[width=0.2\columnwidth]{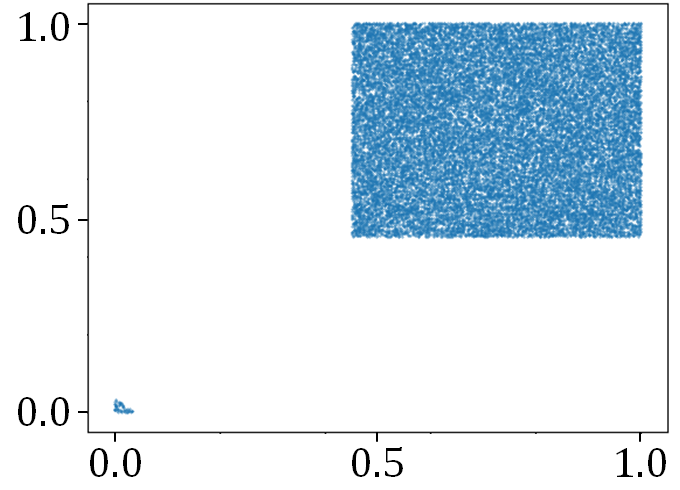}
}
\subfigure[classification rules] { \label{sub_fig:classifyrule}
\includegraphics[width=0.2\columnwidth]{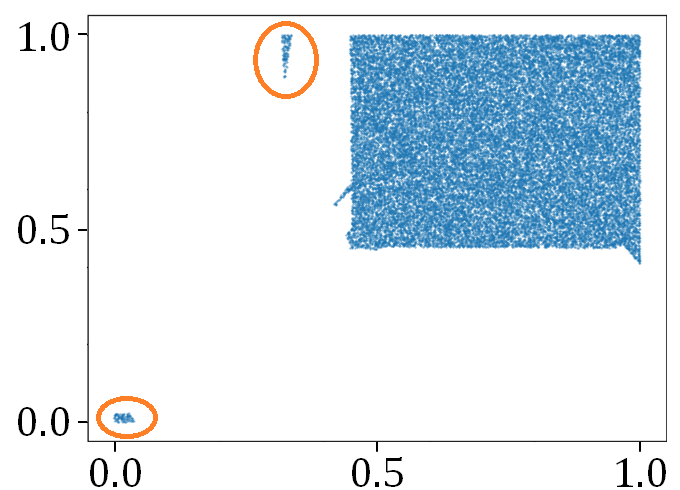}
}
\caption{Synthetic dataset and classification rules}
\label{fig:classification_rules}
\end{wrapfigure}
In this paper, we present a novel global verification framework. To the best of our knowledge,  this is the first work that provides a complete solution to achieving global robust verification for neural networks. Based on the proposed rule-based back-propagation, we analyze the relationship between activation patterns and classification rules, and thus design a new network SDN. Together with a region-based global robustness analysis approach, the verification can be finished in an acceptable duration, dramatically reducing computational complexity. Our evaluation shows that the SDN can perform comparable with classic DNNs, and favourable, the global verification can be achieved, which is unattainable in the past. We hypothesise that SDN is suitable for safety-critical fields, especially for some classification tasks which are not very complex but eager for strict robustness. Developing verification framework for large-scale networks is our further research direction.  
\section{Appendix}
We assign serial numbers to each activation pattern to store the mapping result in a B+ tree for the convenience of global verification:
\begin{align*}
    \Theta.number = &\sum_{h}^{n-1}  ((g_h * group\_num_h + {g'}_h -1)\ \textbf{if}\ (g_h < {g'}_h)
    &\textbf{else}\ (g_h * group\_num_h + {g'}_h))\\
    &*\Pi_{j=h+1}^{n}|\Delta'.j| + (g_n * group\_num_n + {g'}_n)
\end{align*}
where \(n\) is the number of layers, \(h\) represents \(Layer_h\), \(g_h\)(\({g'}_h\)) is the index of \(Layer_h\)'s active(inactive) door, and \(group\_num_h\) is the number of groups in  \({L'}_h\). Two activation patterns \(\Theta_1\) and \(\Theta_2\) satisfy \(\Theta_1.number\ <\ \Theta_2.number\) iff
\begin{align*}
    &(\exists_k\forall_{h(h<k)}\Theta_1.g_h = \Theta_2.g_h \wedge \Theta_1.{g'}_h = \Theta_2.{g'}_h) \bigwedge\\
    &\{\Theta_1.g_k <  \Theta_2.g_k \vee  (\Theta_1.g_k =   \Theta_2.g_k \wedge \Theta_1.{g'}_k <  \Theta_2.{g'}_k)\}
\end{align*}
To prove Theorem~\ref{theorem:common}, we need Lemma~\ref{lemma:share} and Lemma~\ref{lemma:explicit}.
\begin{lemma}\label{lemma:share}
Given two vertexes \(v_i\) and \(v_j (j>i)\), they are adjacent iff they share a common boundary.
\end{lemma}

\begin{proof}\label{proof:share}
Given two adjacent vertexes \(v_i\) and \(v_j(j>i)\), the corresponding classifications are \(S_i\) and \(S_j\). According to the definition of adjacent in Definition~\ref{def:graph}, there is a \(x_0\) satisfies
\begin{align*}
    \exists_{\varepsilon>0}\{[\forall_{x'\in\ I}( x'\in S_i \vee x'\in S_j)]
    \wedge  (I \cap S_i \neq \emptyset) \wedge  (I \cap S_j \neq \emptyset)\}
\end{align*}
where \(I =\{x'|\ ||x'-x_0||_\infty<\varepsilon\}\).

There must be a \(x\) in \(I\) on a boundary \(F\) which satisfies \(\exists_{\varepsilon} \forall_{\mathbb{H}_i\neq \mathbb{H}} C \cap \mathbb{H}_i = \emptyset\) where \(C = \{x'|\ ||x'-x||_\infty<\varepsilon_1\}\), \(\mathbb{H}\) and \(\mathbb{H}_i\) are the point set of boundary \(F\) and \(F_i\)s.
Otherwise we can find a series of points \(u_i (i\geqslant0)\) and corresponding balls \(I_i = \{x|\ ||x-u_i||_\infty<\sigma_i\}\) satisfying:
\begin{itemize}
\item \(u_i\) is on the boundary \(F_i\) and  \(\forall_{\varepsilon_i} \mathbb{H}_{i+1} \cap I_i \neq \emptyset \)
\item \(u_{i+1} \in I_i\), \(u_{i+1}\) on \(F_{i+1}\), \(I_{i+1} \subseteq I_i\) and \(I_{i+1} \cap \mathbb{H}_{i} = \emptyset\)
\end{itemize}
Since we have only finite number of boundaries, here comes the contradiction.
Hence we can find a \(u_k\) in \(I\) on \(F\) which is a boundary of both \(v_i\) and \(v_j\) and a corresponding \(I_k\) satisfies the above condition. As \(u_k\) is on the boundary, it satisfies
\(\forall_{x\in I_k}\ ( x\in S_i \vee x\in S_j)
    \wedge  I_k \cap S_i \neq \emptyset \wedge I_k \cap S_j \neq \emptyset ]\)
Thus \(u_k\) satisfies
\begin{align*}
    \exists_{\varepsilon_k>0} [ \forall_{\mathbb{H'}\neq\mathbb{H}_k}\ I_k \cap \mathbb{H'} = \emptyset\ \wedge\ \forall_{x\in I_k}\ ( x\in S_i \vee x\in S_j)
    \wedge  I_k \cap S_i \neq \emptyset \wedge I_k \cap S_j \neq \emptyset ]
\end{align*}

and \(F\) is a common boundary shared by \(v_i\) and \(v_j\).

The sufficient condition is obvious.
\end{proof}

\begin{lemma}\label{lemma:explicit}
Given two adjacent vertexes \(v_i\) and \(v_j(j>i)\), the shared common boundary of them comes from the \(explicit\ rules\) of \(v_i\).
\end{lemma}

\begin{proof}\label{proof:explicit}
The activation pattern \(\Theta_i\) and \(\Theta_j\) of \(v_i\) and \(v_j\) satisfy \(\Theta_i.number=i \ >\ \Theta_j.number=j\) thus
\begin{align*}
    &\forall_{t(t<k)}\Theta_i.g_t = \Theta_j.g_t \wedge \Theta_i.{g'}_t = \Theta_j.{g'}_t \bigwedge\\
    &\{\Theta_i.g_k <  \Theta_j.g_k \vee  (\Theta_i.g_k =   \Theta_j.g_k \wedge \Theta_i.{g'}_k <  \Theta_i.{g'}_k)\}
\end{align*}
This indicates the first change of door happens on \({L'}^k\).  The change of activation pattern leads to the ``mutation'' of neurons in \({L'}^t(t>k)\). However the change of neurons in \({L'}^t(t\leq k)\) is continuous. According to the proof of Lemma~\ref{lemma:share}, we can find a path in \(D_{\alpha}\) which crosses and only crosses the \(F\). Thus When the point on this path approaches the boundary, there is at most one inactive neuron corresponding to \(F\) approaches 0 in \({L'}^t(t\leq k)\). Here we consider where the \(F\) comes from:
\begin{itemize}
\item \(F\) comes from \(Layer_t(t>k)\). The change of sign of \({x'}_l^t\) will not influence the activation pattern in \({L'}^k\) which contradicts to ``the first change of door happens on \({L'}^k\)''

\item \(F\) comes from  \(Layer_t(t<k)\). If it changes the activation pattern, it contradicts to ``the first change of door happens on \({L'}^k\)''. Otherwise, there must be another inactive neuron in \({L'}^k\) changes the sign at the same time which contradicts to ``there is at most one inactive neuron corresponding to \(F\) approaches 0 in \({L'}^t(t\leq k)\)''

\item \(F\) comes from \(implicit\ rules\) in \(Layer_k\). If the activation pattern does not change, here comes the contradiction to ``the first change of door happens on \({L'}^k\)''. If the activation pattern changes, according to the definition of \(implicit\ rules\), either the index of \emph{active door} \(g_k\) or \emph{inactive door} \({g'}_k\) will become smaller, that is to say, the result is a activation pattern \(v_u(u<i)\) instead of \(v_j\). Thus we have a contradiction
\end{itemize}
We eventually find out that \(F\) comes from \(explicit\ rules\) of \(v_i\) by a process of elimination.

\end{proof}
Based on Lemma~\ref{lemma:share} and Lemma~\ref{lemma:explicit}, the correctness of Theorem~\ref{theorem:common} is obvious. 



\bibliographystyle{plain}
\bibliography{nips}

\end{document}